\documentclass{article}
\usepackage{ijcai24}

\usepackage{times}
\usepackage{soul}
\usepackage{url}
\usepackage[hidelinks]{hyperref}
\usepackage[utf8]{inputenc}
\usepackage[small]{caption}
\usepackage{graphicx}
\usepackage{amsmath}
\usepackage{amsthm}
\usepackage{booktabs}
\usepackage{algorithm}
\usepackage{algorithmic}
\usepackage[switch]{lineno}

\usepackage{hyperref}
\usepackage{stmaryrd}
\usepackage{multirow}
\usepackage{amssymb}
\usepackage{bm}

\newtheorem{theorem}{Theorem}

\title{Large Language Model-Enhanced Algorithm Selection: \\Towards Comprehensive Algorithm Representation}

\author{
Xingyu Wu$^1$
\and
Yan Zhong$^2$\and
Jibin Wu$^1$ \thanks{Corresponding author: Jibin Wu.} \and
Bingbing Jiang$^3$\And
Kay Chen Tan$^1$
\affiliations
$^1$The Hong Kong Polytechnic University\\
$^2$Peking University\\
$^3$Hangzhou Normal University
\emails
\{xingy.wu,jibin.wu,kaychen.tan\}@polyu.edu.hk,
zhongyan@stu.pku.edu.cn,
jiangbb@hznu.edu.cn
}

\begin{document}

\maketitle

\begin{abstract}
    Algorithm selection, a critical process of automated machine learning, aims to identify the most suitable algorithm for solving a specific problem prior to execution. Mainstream algorithm selection techniques heavily rely on problem features, while the role of algorithm features remains largely unexplored. Due to the intrinsic complexity of algorithms, effective methods for universally extracting algorithm information are lacking. This paper takes a significant step towards bridging this gap by introducing Large Language Models (LLMs) into algorithm selection for the first time. By comprehending the code text, LLM not only captures the structural and semantic aspects of the algorithm, but also demonstrates contextual awareness and library function understanding. The high-dimensional algorithm representation extracted by LLM, after undergoing a feature selection module, is combined with the problem representation and passed to the similarity calculation module. The selected algorithm is determined by the matching degree between a given problem and different algorithms. Extensive experiments validate the performance superiority of the proposed model and the efficacy of each key module. Furthermore, we present a theoretical upper bound on model complexity, showcasing the influence of algorithm representation and feature selection modules. This provides valuable theoretical guidance for the practical implementation of our method.
\end{abstract}

\section{Introduction}

Performance complementarity, a phenomenon where no single algorithm consistently outperforms all others across diverse problem instances, is a well-established reality in the realm of optimization and learning problems \cite{kerschke2019automated}. Over the past few years, the growing interest in automated algorithm selection techniques has become evident. These techniques aim to tackle the challenge of selecting the most appropriate algorithm from a predefined set for a given problem instance automatically \cite{ruhkopf2022masif,heins2023study}. Most existing techniques rely on two sources of information: (1) the features of each problem instance and (2) the historical performance of various algorithms across problem instances \cite{pio2023review}. Then, regression \cite{xu2008satzilla} or ranking \cite{abdulrahman2018speeding} based machine learning models are used to establish a mapping from problem features to algorithm performance or the best-performing algorithm. Additionally, there are also methods based on collaborative filtering \cite{misir2017alors} or instance similarity \cite{amadini2014sunny}.

Many contemporary algorithm selection methodologies predominantly treat algorithms as black boxes, concentrating primarily on the features of given problems. Nonetheless, it is essential to acknowledge that, for the majority of scenarios, algorithm feature constitutes a vital source of information, and overlooking it may lead to performance degradation \cite{tornede2020extreme}. Relying exclusively on problem features limits the model to learning a unidirectional mapping from the problem to the algorithm, which does not align with the potentially bidirectional nature of the problem-algorithm relationships. An ideal algorithm selection model should not only identify the suitable algorithm for a given problem but also understand the types of problems for which an algorithm is well-suited, i.e., the interactive synthesis of problem representation and algorithm representation.

\begin{figure*}[t]
\begin{center}
\includegraphics[width=1.00\textwidth]{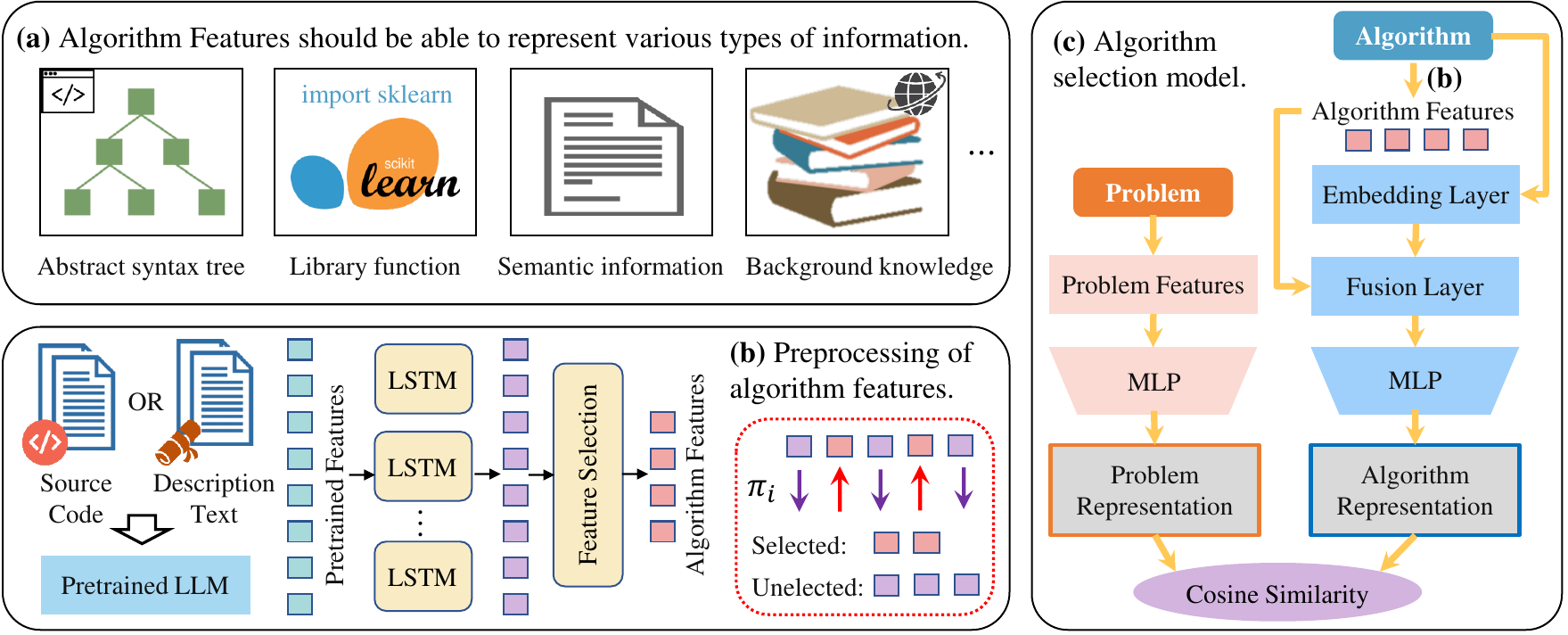}
\end{center}
\caption{The framework of AS-LLM: (a) illustrates the various types of information that influence algorithm representation. (b) shows the algorithm features extracted by LLM and their preprocessing using feature selection. (c) outlines the proposed algorithm selection framework.}
\label{Fig_Framework}
\end{figure*}

The limited consideration given to algorithm features is primarily attributed to the following practical challenges. Firstly, the inherent complexity and diversity of algorithms make it difficult to develop a universally applicable method for algorithm representation. For instance, among the limited studies related to algorithm features, the majority of them manually design features pertaining to specific scenarios or algorithm types, such as features based on hyperparameters \cite{hough2006modern,tornede2020extreme}, execution process \cite{de2021explorative}, model structures, or loss functions \cite{hilario2009data}. These features often have different meanings and scales even within the same category of algorithms, rendering them impractical for universal adaptation across diverse algorithm types. On the other hand, quantifying and elucidating algorithm features indeed presents a formidable challenge. The intricate logic and rich contextual information make it difficult to devise a feature extraction method that captures key algorithmic information. One notable endeavor involves using macro statistical features (e.g., lines of code) and abstract syntax tree (AST) features (e.g., node/edge count) of code to represent algorithms \cite{pulatov2022opening}. While these features can characterize the structural aspects of algorithms, they fall short of directly encapsulating semantic relationships and contextual information within the code. Additionally, certain core algorithms often depend on library functions, which exist at a higher level of abstraction and cannot be directly manifested in the AST. This discrepancy results in the masking of critical information.

With the advent of the pretrained Large Language Models (LLMs) era \cite{ouyang2022training,wu2024evolutionary}, the extraction of algorithm features from code-related text has become significantly more attainable. LLMs exhibit three notable advantages in algorithm representation: (1) Rich and robust information content: LLMs not only encapsulate the syntax structure of code represented by AST but also capture nuanced semantic information and exhibit contextual awareness \cite{guo2020graphcodebert}, as shown in Figure \ref{Fig_Framework}(a), which allows algorithm features to maintain robustness across diverse programming languages. (2) Vast knowledge repository: LLMs leverage extensive knowledge from large-scale code repositories, enabling them to represent common patterns, library function usage, and module relationships \cite{chen2021evaluating}. This imparts superior domain adaptation capabilities to LLMs in the algorithm selection tasks. (3) Automation and universality: LLMs obviate the need for manual feature design as existing methods, but automatically extract features from code text, demonstrating universal applicability across various scenarios and algorithms. Even in scenarios where code is unavailable, LLMs can extract pertinent information from algorithm description text (e.g., pseudo-code).

Leveraging the powerful representation capability of LLMs, we introduce the proposed Algorithm Selection model based on LLM (AS-LLM). With the support of pretrained LLMs for code or text (e.g., \cite{chen2021evaluating}), we efficiently represent the algorithm with minimal training overhead. Simultaneously, it is worth noting that LLM-based representations possess high dimensions \cite{chen2021evaluating}, and not all features are directly relevant to algorithm selection tasks. To further enhance the model, we integrate a feature selection module depicted in Figure \ref{Fig_Framework}(b), which carefully scrutinizes the pretrained features and extracts a subset of algorithm features that are most significant by deriving importance weights for them. Moreover, AS-LLM does not directly concatenate algorithm and problem features for performance regression. Instead, it adopts distinct treatment methodologies for these two feature types, as shown in Figure \ref{Fig_Framework}(c). Independent feature processing networks are employed to obtain problem and algorithm representations of equal length separately, which could mitigate premature interference between the two feature types. These representations are then intricately fused through similarity calculation modules, ultimately guiding the selection decision based on the matching degree between a given problem and different algorithms. The key contributions of this study are summarized as follows:

\begin{itemize}
\item To the best of our knowledge, this study represents a pioneering application of LLM in algorithm selection tasks, leveraging the powerful representation capability of LLM to extract discriminative algorithm features. We also introduce an algorithm feature selection module to identify critical features for algorithm selection.
\item The comprehensive algorithm representation bestows AS-LLM with at least three advantages: (i) A more nuanced modeling of the bidirectional nature of algorithm selection tasks; (ii) The generalization capability to novel algorithms not encountered during training; (iii) Robust performance superiority in different scenarios.
\item We not only highlight the performance superiority of AS-LLM through empirical studies but also provide a rigorous upper bound on its model complexity. This theoretical validation serves as a fundamental basis for model design and practical implementation.
\end{itemize}

\section{Related Work}
\label{background}

In this section, Section 2.1 introduces the task definition of algorithm selection and provides a comprehensive review of algorithm selection techniques based on problem features, where a clear distinction is made between feature-based and feature-free approaches. In Section 2.2, a detailed examination of existing algorithm feature-based algorithm selection techniques is presented, and the disparities between probing features and algorithm features are also highlighted.

\subsection{Algorithm Selection}

Algorithm selection aims to choose the appropriate algorithm for each problem instance from a set of algorithms \cite{rice1976algorithm}. Traditionally, the problem of per-instance algorithm selection can be defined as follows: Given a problem set $P$, an algorithm set $A$ for solving problem instances in $P$, and a performance metric $\mathcal{M}$: $P\times A \rightarrow \mathbb{R}$ which quantifies the performance of any algorithm $a\in A$ on each instance $p\in P$, per-instance algorithm selection should construct a selector $\mathcal{S}: P \rightarrow A$ that assigns any problem instance $p\in P$ to an algorithm $a\in A$, optimizing the overall performance expectation $\mathbb{E}[\mathcal{M}(p,\mathcal{S}(p))]$ on $P$ according to the metric $\mathcal{M}$ \cite{kerschke2019automated}. In the following, we first review the classical types of algorithm selection techniques, followed by a focus on the study of algorithm features in existing works.

Regression-based techniques \cite{xu2008satzilla} aim to predict the performance of different algorithms based on a set of features. Similarly, ranking-based approaches \cite{cunha2018label,abdulrahman2018speeding} assign a rank or score to each algorithm, indicating its relative suitability for a given problem. The ranking can be determined using various methods, such as pairwise comparisons or learning-to-rank algorithms. Some studies formalize algorithm selection as a collaborative filtering problem \cite{misir2017alors,fusi2018probabilistic} and utilize a sparse matrix with performance data of only a few algorithms on each problem instance. Similarity-based approaches \cite{amadini2014sunny,kadioglu2010isac} select algorithms based on the similarity between the current problem instance and previously encountered instances. Moreover, hybrid methods \cite{hanselle2020hybrid,fehring2022harris} combine multiple techniques above and leverage their strengths to improve the accuracy and robustness of algorithm selection. For more detailed information, please refer to \cite{kerschke2019automated}.

\paragraph{Difference between Feature-based and Feature-free Algorithm Selection:} The AS-LLM proposed in this paper falls under the feature-based category. The distinction between feature-based and feature-free algorithm selection \cite{alissa2023automated} lies in two aspects: 1) Feature-free methods aim to avoid manual crafting of features by designing models that automatically extract problem representations, while feature-based methods rely on pre-calculated feature profiles. 2) Feature-free models are typically specific to certain problem types and may not be universally applicable across a wide range of scenarios, as the problem representation module need to be tailored to each specific scenario, such as traveling salesman problem \cite{zhao2021towards} or continuous optimization problem \cite{seiler2024deep}. Conversely, feature-based methods only focus on the algorithm selection task itself. These models are generally applicable to various scenarios where problem features have already been extracted, such as scenarios in ASlib. Due to their characteristics, feature-free methods are usually evaluated in a specific scenario, while feature-based methods need to be validated across diverse scenarios, as shown in Section 5.

\subsection{Algorithm Feature-based Algorithm Selection}

Currently, most research focuses on using problem features to achieve algorithm selection. Although there is relatively little literature on using algorithm features, there have been some attempts made in the following works.

\citeauthor{hough2006modern} are among the early researchers to consider the measurement of algorithm features. They extract hyperparameters of optimization algorithms and additionally construct five features to describe how algorithms handle optimization problems \cite{hough2006modern}. Similarly, \citeauthor{tornede2020extreme} [\citeyear{tornede2020extreme}] discuss the feature representation of machine learning algorithms. They also employ the values of hyperparameters as algorithm features and then concatenate algorithm features and problem features to build a regression or ranking model. These features, along with problem features, are inputted into the decision tree and SVM models to identify the optimal algorithm for a problem instance. \citeauthor{hilario2009data} suggest using information related to the structure, parameters, and cost function of a model as algorithm features without empirical study on this yet \cite{hilario2009data}. Until recently, \citeauthor{pulatov2022opening} [\citeyear{pulatov2022opening}] utilize macro-level features of code (e.g., lines of code and cyclomatic complexity), as well as abstract syntax tree features (e.g., the number of nodes and edges), to describe algorithm characteristics. This allows for the development of a unified model that can regress the performance of all algorithms, instead of building a separate regression model for each algorithm, as in their previous work. Moreover, \citeauthor{de2021explorative} [\citeyear{de2021explorative}] design a method for extracting time series features specifically for algorithm representation of Covariance Matrix Adaptation Evolution Strategy and its variants \cite{hansen2001completely}. As analyzed in Section 1, these manually designed features either lack generality and only serve specific algorithm selection scenarios, or fail to fully capture the entire information of algorithms. This paper will introduce a novel algorithm representation technique based on the LLM and design a new framework for algorithm selection.

\paragraph{Difference between Probing Features and Algorithm Features:} To underscore the significance of algorithm features, we delve into the distinction between another closely aligned feature known as probing feature \cite{hutter2014algorithm}. Obviously, they are not algorithm features. Instead, they fall under the category of problem features, encompassing high-level statistics derived from short probing runs of a solver on specific problem instances. While its extraction process involves the algorithm execution, these features primarily describe information specific to the corresponding problem instances. In other words, they capture variations observed among different problem instances during short `probing' runs of the same solver. Consequently, probing features are primarily capable of distinguishing between problem instances but do not differentiate between candidate algorithms. On the other hand, the algorithm features discussed in this paper precisely describe the characteristics of the algorithms themselves. By extracting discriminative information from algorithmic code, these features directly reflect the distinctions between algorithms.

\section{Algorithm Selection Based on LLMs}

This section introduces the main details of AS-LLM \footnote{The implementation of AS-LLM is available at \url{https://github.com/wuxingyu-ai/AS-LLM}}. As illustrated in Figure \ref{Fig_Framework}, AS-LLM primarily consists of three modules: the problem representation module, algorithm representation module, and similarity calculation module. Among them, the extraction and utilization of algorithm features are key aspects of this study. For each candidate algorithm, its corresponding code snippet is an easily accessible resource. Even in the absence of open-source code, the algorithm's descriptive text (such as pseudo-code, papers, or manuals) is available and can serve as a viable alternative. Leveraging the code representation capability of pretrained LLMs, we aim to comprehensively extract and characterize the essence of each algorithm from their code-related texts. In this endeavor, the employed LLM can take on multiple forms, as long as its pretraining data contains code text. Hence, it can either be a general-purpose model pretrained on a vast corpus of text data (e.g., GPT-3 \cite{floridi2020gpt}) or a specialized model tailored for code comprehension and generation (e.g., CodeBERT \cite{feng2020codebert}). The model's ability to comprehend code syntax, structure, and semantics allows it to identify critical algorithmic patterns and techniques. This entails recognizing loops, conditionals, data structures, and their interplay within the code.

Specifically, for $\forall a \in A$, algorithm indices are mapped to corresponding continuous vectors, obtained from a pretrained LLM:
\begin{equation}
\mathbf{E}_a = \text{LLM-Embedding}(a) \in \mathbb{R}^{e}
\end{equation}
where $e$ is determined by the output scale of the chosen LLM. These embeddings can be frozen during training to preserve the knowledge encoded within them. When there is enough training data available, parameters in the LLM can be fine-turned together with the other parts of the model, and the embedding layer is updated according to the output of LLM. Given the sequential nature of algorithm features from LLM, we utilize an LSTM to handle these algorithm features:
\begin{equation}
\begin{aligned}
&\mathbf{H}_a = \text{LSTM}(\mathbf{E}_a)\\
&\mathbf{F}_a = \text{Linear}(\text{Concatenate}(\mathbf{h}_{-1}, \mathbf{h}_0))
\end{aligned}
\end{equation}
where $\mathbf{h}_{-1}, \mathbf{h}_0 \in \mathbf{H}_a$ are the LSTM encoder's final hidden state and the initial hidden state. Herein, we chose to use LSTM instead of other advanced models such as Transformer primarily due to the limitations posed by the training data scale in the algorithm selection. While Transformer exhibits strong modeling capabilities for sequential data, it typically requires a larger amount of data for training. Given the finite number of candidate algorithms and the relatively small scale of the problems in algorithm selection tasks, the performance of the Transformer model could be constrained. In contrast, LSTM performs better when handling a small number of training samples.

As mentioned earlier, the representation of algorithm-related text by LLM is typically high-dimensional, which can range in the hundreds or thousands, there exists a disparity in dimensionality between algorithm features and problem features. Among these features, not all of them are relevant to the algorithm selection task. Therefore, it is necessary to add a feature selection module \cite{wu2020multi,wu2022multi} in front of the model's main body to select the predictive features based on the task information. Inspired by \cite{wang2022autofield,wu2019accurate}, we first arrange a parameter $\pi_i$ to the $i$-th features in $\mathbf{F}_a$, which are the same in number as the features and indicate whether a feature should be selected. Then, we use Gumbel distributions to generate samples from classification distributions and make the feature selection operation differentiable. Specifically, for each parameter $\pi_i\in[0,1]$, we sample a random variable $\sigma^{+}_i,\sigma^{-}_i \sim \text{Gumbel}(0,1)$ from the standard Gumbel distribution to obtain noise term. According to its property,
\begin{equation}
\mathbb{P}(\log\pi_i+\sigma^{+}_i>\log(1-\pi_i)+\sigma^{-}_i)=\pi_i.
\end{equation}
Using the softmax function as a continuous and differentiable approximation \cite{jang2016categorical}, the $i$-th weighted feature ${\mathbf{F}^\prime_a}_i$ is:
\begin{equation}
{\mathbf{F}^\prime_a}_i = \frac{\exp\left(\frac{\log\pi_i+\sigma^{+}_i}{\tau}\right)}{\exp\left(\frac{\log\pi_i+\sigma^{+}_i}{\tau}\right) + \exp\left(\frac{\log(1-\pi_i)+\sigma^{-}_i}{\tau}\right)}{\mathbf{F}_a}_i.
\end{equation}
where the weight indicates the probability of the $i$-th features being selected, and the parameter $\tau$ is the temperature parameter. When $\tau$ is greater than 0, it ensures the smoothness of the distribution and facilitates better gradients. As $\tau$ approaches 0, the values of $\mathbb{P}(i)$ converge towards $\pi_i$. Conversely, as $\tau$ approaches infinity, the probability tends toward a uniform distribution. For better practical usage, the feature selection module is trained together with other parts of the model. Based on the obtained values of the parameters $\pi_i$, the module selects the top-$k$ most important features. These selected features are then retained, and the feature selection module is removed. The other parts of the model are subsequently retrained without the feature selection module.

In addition to the pretrained features extracted by LLM, AS-LLM also utilizes an embedding layer to adaptively capture the algorithm representation:
\begin{equation}
\label{eq_5}
\mathbf{E}^\prime_a = \text{Embedding}(a) \in \mathbb{R}^{k}
\end{equation}
where $k$ indicates the same dimension with $\mathbf{F}^\prime_a$. This layer is directly learned from the matching relationship between the algorithm and the problem. Therefore, the embedding layer is only applicable when the candidate algorithm remains unchanged and can accurately represent the algorithms that have appeared in the training data. When the candidate algorithm remains unchanged, the adaptive features will bring performance gains in prediction. The adaptive features and pretrained features are fused in the fusion layer by weighted addition with a parameter $\alpha\in[0,1]$, resulting in an algorithm feature set $\mathbf{V}_a$. The problem features $\mathbf{F}_p$ and the algorithm features will each go through their respective multilayer perceptron (MLP) modules and then undergo a calculation of the cosine distance $d$:
\begin{equation}
\label{eq_6}
\begin{aligned}
&\mathbf{V}_p = \text{MLP}_P(\mathbf{F}_p) \in \mathbb{R}^{m} \\
&\mathbf{V}_a = \text{MLP}_A\left[\alpha\mathbf{F}^\prime_a + (1-\alpha)\mathbf{E}^\prime_a\right] \in \mathbb{R}^{m} \\
&d = \frac{\mathbf{V}_a \cdot \mathbf{V}_p}{\|\mathbf{V}_a\| \cdot \|\mathbf{V}_p\|}
\end{aligned}
\end{equation}
Cosine similarity measures the cosine of the angle between these vectors and provides a score indicating their similarity. A higher $d$ suggests a better match between the problem and the algorithm. Finally, the cosine similarity is concatenated with problem features and algorithm features, and jointly influences the model's final output through an MLP layer:
\begin{equation}
\text{Similarity}(a,p) = \text{MLP}_S\left[\text{Concatenate}(\mathbf{V}_a, \mathbf{V}_p, d)\right].
\end{equation}
The model's output $\text{Similarity}(a,p)$ is used to predict whether the input problem and algorithm are mutually compatible. The $\text{Similarity}(a,p)$ can be understood in different contexts, such as the matching degree or predicted performance, which is contingent upon the training approach of the model.

\section{Upper Bound of the Model Complexity}
\label{theory}

Due to the incorporation of algorithm features in this study, the model unavoidably becomes more intricate in comparison to models relying solely on problem features. Consequently, this section aims to delve into the model complexity and endeavor to extract some theoretical conclusions that can guide real-world applications. It is important to note that, to enable generalization over new candidate algorithms, both the embedding layer and the cosine distance calculation should be excluded from AS-LLM. The rationale for removing the embedding layer has been previously discussed after Eq. (\ref{eq_5}). The elimination of the cosine distance calculation $d$, on the other hand, is primarily intended to establish a theoretical upper bound on the model complexity, which will be elucidated during the derivation process of Theorem 1 in this section. We denote the variant model as AS-LLM$_g$, and analyse its complexity as follows.

Denote the space of AS-LLM$_g$ as $\mathcal{F}$, $\mathbf{f}(x_i)\in\mathcal{F}$ is a real-valued network mapping the instance $x_i$ to $\mathbb{R}$. For analytical convenience, we consider AS-LLM$_g$ as a unified MLP, including all sub-MLPs either for problem representation, algorithm representation, or similarity calculation. This transformation is mathematically equivalent, as all parameters and activation functions in this unified MLP can be derived from the original model without extra computation. For instance, the parameters $W^{(1)}$ of the MLP's first layer can be obtained by combining the parameters $W^{(1)}_p$ and $W^{(1)}_a$ from the $\text{MLP}_P$ and $\text{MLP}_A$ in Eq. (\ref{eq_6}), respectively, i.e., $W^{(1)}=\begin{pmatrix}
    W^{(1)}_a & O \\
    O & W^{(1)}_p \\
\end{pmatrix}$. Similarly, each activation function also corresponds to its counterpart in the original model. Under this simplification, AS-LLM$_g$ is treated as an MLP structure, facilitating the derivation of concise forms for the upper bounds of model complexity. Firstly, we introduce the concept of inductive Rademacher Complexity \cite{koltchinskii2001rademacher} as a measure of the model's complexity, which is calculated as:
\begin{equation}
\label{eq_rc}
\hat{\mathfrak{R}}_{|\mathcal{S}|}\left(\mathcal{F}\right) = \frac{1}{|\mathcal{S}|} \mathbf{E}_{\boldsymbol{\sigma}}\left[\sup _{\mathbf{f} \in \mathcal{F}} \sum_{x_i\in\mathcal{S}} \sigma_i \mathbf{f}\left(x_i\right)\right]
\end{equation}
where $|\mathcal{S}|$ is the scale of training data $\mathcal{S}$, and $\bm{\sigma} = \{ \sigma_i \}_{i=1}^n$ is an independent uniform $\{\pm 1\}$-valued random variables, $\sigma_i=1$ with probability $\frac{1}{2}$ and $\sigma_i=-1$ with the same probability. Inductive Rademacher Complexity is a concept used to quantify the model complexity. By solving for the upper bound of the Rademacher complexity of the model, we can not only provide a theoretical estimation of how the introduction of algorithm features affects the model complexity but also evaluate its ability to fit and generalize to the training data. These theoretical insights are crucial for assessing the model's generalization capability, conducting effective model selection, controlling model complexity, and driving advancements in algorithm design. In the following, we present Theorem 1 to give an upper bound on the model complexity of AS-LLM$_g$.

\begin{theorem}
\label{bound3}
Let $|\mathcal{S}_{\mathcal{P}}|$ and $|\mathcal{S}_{\mathcal{A}}|$ denote the number of problems and algorithms in training samples, $W^{(i)}$ denote the parameter in the $i$-th layer with the upper bound of $F$-norm $R_i$, i.e., $\|W^{(i)}\|_F\leq R_i$, and $l$ denote the number of layers. Then, for the model $\mathbf{f}\left(x_i\right)$ with $1$-Lipschitz, positive-homogeneous activation functions, the inductive Rademacher complexity of AS-LLM$_g$ is bounded by:
\begin{equation}
\label{eq_res}
\begin{aligned}
\hat{\mathfrak{R}}_{|\mathcal{S}|}\left(\mathcal{F}\right) \leq \frac{\sqrt{2l\log 2 \Gamma_{\mathbf{f}} \Gamma_{\mathcal{S}} + 2 |\mathcal{S}_{\mathcal{P}}|^{\frac{1}{2}} |\mathcal{S}_{\mathcal{A}}|^{\frac{1}{2}} \Gamma^2_{\mathbf{f}} \Gamma_{\mathcal{S}}^{\frac{3}{2}}}}{\sqrt{|\mathcal{S}_{\mathcal{P}}| \cdot |\mathcal{S}_{\mathcal{A}}|}}
\end{aligned}
\end{equation}
where $\Gamma_{\mathbf{f}}$ and $\Gamma_{\mathcal{S}}$ are model-related and data-related variables, denoted as:
\begin{equation}
\begin{aligned}
\Gamma_{\mathbf{f}} = \prod_{i=1}^l R_i, \ \  \Gamma_{\mathcal{S}} = \max_{x_j\in \mathcal{S}}\| x_j \|^2.
\end{aligned}
\end{equation}
\end{theorem}
\begin{proof}
Due to the constraints of space, we present a condensed outline of the proof here. We start by representing Eq. (\ref{eq_rc}) as the computation performed in the final layer of AS-LLM$_g$ as:
\begin{equation}
\label{eq_ini}
\begin{aligned}
\frac{1}{|\mathcal{S}|} \mathbf{E}_{\boldsymbol{\sigma}} \sup _{\mathcal{N}_1^{l-1}, W^{(l)}} \sum_{x_i\in\mathcal{S}} \sigma_i W^{(l)} \phi^{(l-1)}\left(\mathcal{N}_1^{l-1}\left(x_i\right)\right)
\end{aligned}
\end{equation}
where $|\mathcal{S}| = |\mathcal{S}_{\mathcal{P}}| \cdot |\mathcal{S}_{\mathcal{A}}|$, $\mathcal{F}$ is the model space of AS-LLM$_g$, $\mathcal{N}_i^{j}$ denotes the sub-network from $i$-th layer to $j$-th layer in the model $\mathbf{f}\left(x_i\right)$, $W^{(i)}$ denotes the parameters of the $i$-th layer and $\phi^{(i)}(*)$ denotes the $1$-Lipschitz and positive-homogeneous activation function.

By repeatedly utilizing Lemma 1 from the literature \cite{golowich2018size} to expand the upper bound of Eq. (\ref{eq_ini}), we can gradually convert the complexity of the model into an expression that is solely composed of the input instances, without any involvement of intermediate layer outputs in the model:
\begin{equation}
\label{eq_CalR}
\hat{\mathfrak{R}}_{|\mathcal{S}|}\left(\mathcal{F}\right) \leq \frac{1}{|\mathcal{S}|} \log \left[2^l \mathbf{E}_{\boldsymbol{\sigma}} \exp \left(\prod_{i=1}^l R_i \left\|\sum_{x_j\in\mathcal{S}} \sigma_j x_j\right\|\right)\right].
\end{equation}

Based on Eq. (\ref{eq_CalR}), we introduce a parameter $\lambda$ and break down the mathematical expectation expression into two parts that can be proven to have upper bounds respectively. The purpose of this decomposition is to construct a form suitable for the Bounded Differences Inequality \cite{mcdiarmid1989method}. Hence, one of the split parts can be bounded tightly using the Bounded Differences Inequality, while the other part requires the application of the Jensen's Inequality to establish an upper bound. By reassembling the two parts obtained and treating them as a function of the parameter $\lambda$, we minimize the expression and ultimately derive the concise upper bound stated in Eq. (\ref{eq_res}) of Theorem 1.
\end{proof}

Based on Theorem 1, we observe that incorporating algorithm features into the model does not necessarily result in a significant increase in model complexity. Eq. (\ref{eq_res}) reveals that factors contributing to increased model complexity primarily include $\Gamma_{\mathcal{S}}$ and $\Gamma_{\mathbf{f}}$. When algorithm features are introduced, the value of $\Gamma_{\mathcal{S}}$ increases as $\| x_j \|^2$ grows. Simultaneously, a more complex model leads to higher values of $R_i$ (at least in the input layer), thereby increasing $\Gamma_{\mathbf{f}}$. This analysis underscores the effectiveness of the feature selection module in controlling model complexity and emphasizes that incorporating algorithm feature selection enhances the model's generalization since the upper bound of $\hat{\mathfrak{R}}_{|\mathcal{S}|}\left(\mathcal{F}\right)$ can further lead to the bound of generalization error.

Moreover, it is worth noting that when the number of training instances is large, the increase in $|\mathcal{S}_{\mathcal{A}}|$ and $|\mathcal{S}_{\mathcal{P}}|$ effectively reduces model complexity at a rate of approximately $|\mathcal{S}_{\mathcal{P}}|^{\frac{1}{4}}|\mathcal{S}_{\mathcal{A}}|^{\frac{1}{4}}$. Notably, the introduction of algorithm features plays a positive role in reducing model complexity, especially when considering the increased number of algorithms. Hence, Theorem 1 also suggests that the AS-LLM algorithm, based on algorithm features, is particularly suitable for scenarios with a higher number of algorithms, as a larger $|\mathcal{S}_{\mathcal{A}}|$ effectively controls model complexity and improves generalization capacity.

\section{Experiments}

In the following, we examine the results of the empirical studies. Firstly, we will introduce the experimental setup, which encompasses the ASlib benchmark, the comparison algorithms, and the evaluation metrics. Building upon this foundation, we present the experimental results for 10 ASlib scenarios in Section 5.2. Furthermore, in Section 5.3, we provide the results of the ablation experiments.

\begin{table}[t]
\centering
\setlength{\tabcolsep}{0.5mm}
\resizebox{0.47\textwidth}{!}{%
\begin{tabular}{ccccc}
\hline
ASLib Scenario & \#Problem   & \#Algorithm    & \#Feature  & Cutoff Time \\ \hline
BNSL-2016  & 1179 & 8 & 86  & 7200 \\
CSP-MZN-2013  & 4642  & 11  & 155 & 1800 \\
GLUHACK-18  & 353 & 8 & 50 & 5000 \\
MAXSAT-PMS-2016  & 601 & 19 & 37 & 1800 \\
MAXSAT15-PMS-INDU  & 601 & 29 & 37 & 1800 \\
PROTEUS-2014  & 4021 & 22 & 198 & 3600 \\
QBF-2016  & 825 & 24 & 46 & 1800 \\
SAT03-16-INDU  & 2000 & 10 & 483 & 5000 \\
SAT18-EXP  & 353 & 37 & 50 & 5000 \\
TSP-LION2015  & 31060 & 4 & 122 & 3600 \\
\hline
\end{tabular}%
}
\caption{The statistical property of experimental benchmarks.}
\label{tab_dataset}
\end{table}

\begin{table*}[t]
\centering
\setlength{\tabcolsep}{3mm}
\resizebox{1,0\textwidth}{!}{%
\begin{tabular}{c|cc|cccccc}
\hline
Scenario & VBS   & SBS   & ISAC   & MCC    & SATzilla11   & SNNAP   & SUNNY    & AS-LLM    \\ \hline
BNSL-2016  & 211.09  & 8849.64 & 6317.29 & 3786.08 & 2004.23 & 37366.23 & 3860.89 & \textbf{1385.58}  \\
CSP-MZN-2013  & 239.60 & 6234.09 & 2181.34 & 2072.73 & 1444.32 & 16857.15 & \textbf{879.25} & 2790.79  \\
GLUHACK-18  & 906.39 & 17829.03 & 15558.97 & 8098.13 & 8113.07 & 20437.53 & 10513.81 & \textbf{7247.89}  \\
MAXSAT-PMS-2016  & 37.76 & 697.02 & 962.44 & 1051.33 & 562.27 & 15692.93 & 476.75 & \textbf{367.85}  \\
MAXSAT15-PMS-INDU  & 54.20 & 1057.40 & 985.36 & 1260.90 & 470.48 & 16527.11 & 411.91 & \textbf{365.98}  \\
PROTEUS-2014  & 170.89 & 10242.71 & 6890.79 & 9109.51 & \textbf{6508.37} & 22871.71 & 6518.17 & 7263.73  \\
QBF-2016  & 17.87 & 3292.44 & 2106.43 & 2333.24 & 2780.95 & 11310.70 & 2101.89 & \textbf{1851.23}  \\
SAT03-16-INDU  & 823.02 & 5198.46 & 4725.42 & 4466.92 & 4266.63 & 10406.79 & 4464.43 & \textbf{4123.52}  \\
SAT18-EXP  & 1705.43 & 11945.73 & 10390.83 & 9760.77 & 8949.86 & 44091.95 & 8112.77 & \textbf{7796.62}  \\
TSP-LION2015  & 44.63 & 189.65 & 1118.56 & 2673.78 & 1839.47 & 10218.48 & 778.08 & \textbf{345.37}  \\
\hline
\end{tabular}%
}
\caption{Evaluation Results on ASlib Benchmarks.}
\label{tab_evaluation}
\end{table*}

\subsection{Experiment Settings}

\paragraph{Dataset Description:} ASlib (Algorithm Selection Library) Benchmark\footnote{\url{https://www.coseal.net/aslib/}} \cite{bischl2016aslib} is a standardized dataset for algorithm selection problems, aimed at providing a common benchmark for evaluating and comparing the performance of different algorithm selection methods. The ASlib Benchmark includes problem instances from multiple domains and their associated algorithm performance data. Each problem instance is described by a set of features, such as the properties and constraints of the problem. The algorithm performance data associated with it consists of a set of algorithms and their performance measurements on each problem instance, such as runtime or solution quality. In line with the experimental setup of most existing studies, this paper selects 10 scenarios from ASlib, covering different domains and scales, to validate the effectiveness of the proposed algorithms. These datasets come from different scenarios and encompass algorithms and problem scales of various magnitudes, whose statical information is shown in Table \ref{tab_dataset}.

\paragraph{Comparing Approaches:} Five classic algorithm selection methods, including Instance-Specific Algorithm Configuration (ISAC) \cite{kadioglu2010isac}, Multi-Class Classification algorithm selection (MCC) \cite{xu2011hydra}, SATzilla11 \cite{xu2011hydra}, Solver-based Nearest Neighbor for Algorithm Portfolios (SNNAP) \cite{collautti2013snnap}, and SUb-portfolio Selection using Nearest Neighbor in a lazY manner (SUNNY) \cite{amadini2014sunny}, participate in the comparative experiments in this paper. Additionally, the two representative performance, virtual best solver (VBS) and single best solver (SBS), are also compared in the experimental results. VBS represents the most ideal solver, capable of selecting the optimal algorithm for each problem instance. SBS, on the other hand, is the most direct and simplest solver, selecting the solver with the overall best performance without distinguishing between problem instances. By comparing the algorithm selection methods with SBS and VBS, we can assess the performance more intuitively. The ideal performance should lie between the performances of SBS and VBS, and the closer the performance is to that of VBS, the better the method is considered.

\paragraph{Evaluation metrics:} Most scenarios in ASlib focus on the algorithm's solution time, and Penalized Average Runtime at a factor of 10 (PAR10 score) is widely used as the performance metric in algorithm selection research. Specifically, the PAR10 score of instance $p$ is calculated as follows:
\begin{equation}
\operatorname{PAR} 10(p)=\left\{\begin{array}{ll}
t_p & \text { if } t_p \leqslant C \\
10 \cdot C & \text { else }
\end{array} .\right.
\end{equation}
For each problem instance $p$, the actual running time $t_p$ of the selected algorithm is compared to a predetermined cutoff time $C$, provided in Table \ref{tab_dataset}. If the selected algorithm finds a solution within the cutoff time, the actual running time is recorded. Otherwise, it incurs a penalty of 10 times the cutoff time $10\cdot C$. Finally, the PAR10 score is obtained by averaging the results across all problem instances. PAR10 score takes into account both the algorithm's solution time and timeout situations. A lower PAR10 score indicates a more effective algorithm selection method.

\subsection{Performance Comparison}

While all comparison algorithms utilize problem features as input, AS-LLM uniquely incorporates algorithm code or relevant text, which are processed by the respective LLMs (UniXCoder \cite{guo2022unixcoder} and BGE \cite{xiao2023c}) to extract algorithm features. The training process of AS-LLM involves two stages. Firstly, we train the complete model, which includes the feature selection module. Using the learned parameters from the feature selection module, we determine the dimensions to be retained. Subsequently, we retrain the main model using these selected algorithm features. The main model can be obtained through various training approaches, such as constructing a regression model to predict algorithm performance, or building a classification model to assess the compatibility between algorithms and specific problem types, or directly ranking the degree of match between a certain problem and different algorithms. In this study, we primarily use regression modeling for performance prediction and classification modeling for matching/mismatching identification, and record the superior results obtained from both approaches.

To ensure a fair comparison between AS-LLM and existing classical approaches, we adopt a consistent procedure. Prior to each experiment, we randomly select $80\%$ of the problems as the training set, while the remaining $20\%$ constitute the test set. For AS-LLM, our process involves selecting the appropriate LLM based on the algorithm's source text, followed by fine-tuning the network's hyperparameters. This includes adjusting the number of layers in each MLP, choosing suitable activation functions, determining the number of neurons in each hidden layer of the components, adjusting the parameters in feature selection module, and determining the incorporation of the MLP module after similarity calculation. Additionally, we explore the fusion of two distinct algorithm representation vectors, i.e., the LLM representation and embedding layer representation, and assign pre-defined weights $\alpha$ in Eq. (\ref{eq_6}) before conducting the experiments. We record and analyze the optimal results achieved across various hyperparameter configurations. For the comparative algorithms, we follow a similar approach in adjusting the hyperparameters. The evaluation results, presented in Table \ref{tab_evaluation}, showcase the PAR10 scores for each comparing method, where the best results for each scenario are highlighted in boldface.

Analyzing Table \ref{tab_evaluation}, it is evident that AS-LLM outperforms other methods in eight out of the ten datasets, while still exhibiting competitive performance in the CSP-MZN-2013 and PROTEUS-2014 scenarios. These results highlight the significant enhancement in algorithm selection effectiveness achieved by AS-LLM through its consideration of algorithm features. The incorporation of algorithm features in AS-LLM provides supplementary information, which is effectively extracted by the LLM and embedding layers. Consequently, AS-LLM trains a highly efficient model that capitalizes on the bidirectional adaptation between problems and algorithms. Due to the nature of the algorithm selection task and the PAR10 metric, all methods exhibit large variances in PAR10 scores across different sampled data. Nevertheless, AS-LLM still demonstrates generally stable performance. It should be noted, however, that AS-LLM exhibits some limitations in the CSP-MZN-2013 and PROTEUS-2014 scenarios. This can be attributed to the unavailability of complete algorithm code files in these scenarios and the insufficient description of algorithm characteristics through code-related text, resulting in inadequate model training.

\subsection{Ablation Study}

\begin{figure*}[t]
\begin{center}
\includegraphics[width=0.84\textwidth]{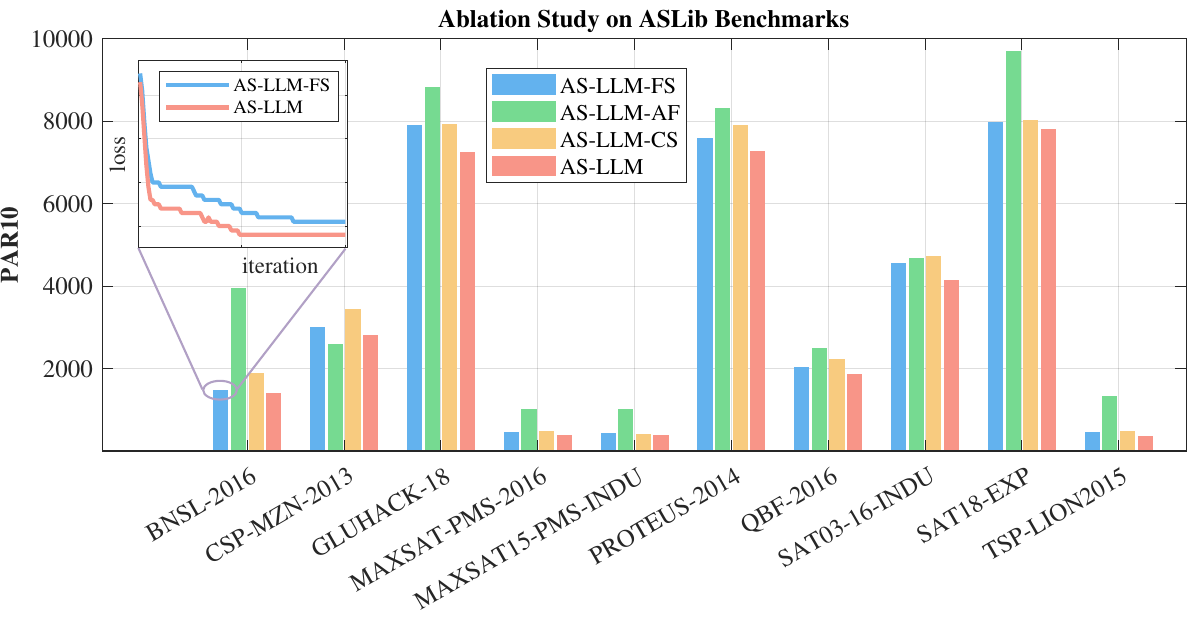}
\end{center}
\caption{Ablation study on ASlib benchmarks.}
\label{Fig_Ablation}
\end{figure*}

To investigate the performance improvements brought by the key modules in AS-LLM, we conduct ablation experiments in this section, including AS-LLM-AF, AS-LLM-FS, and AS-LLM-CS. AS-LLM-AF removes the algorithm feature computation module from AS-LLM, relying solely on problem features for algorithm selection. AS-LLM-FS eliminates the feature selection process in AS-LLM to assess its impact on the model. AS-LLM-CS excludes the cosine similarity calculation module, using only MLP to derive the final results from the algorithm and problem representations. The results of the ablation experiments are presented in Figure \ref{Fig_Ablation}.

The findings from Figure \ref{Fig_Ablation} reveal that, except for the CSP-MZN-013 scenario, AS-LLM outperforms all comparative methods, indicating the positive effects of the tested modules on the model. Notably, AS-LLM-AF exhibits the largest performance loss, underscoring the crucial role of algorithm features in algorithm selection and validating the core innovation of this study. In the CSP-MZN-013 scenario, the superiority of AS-LLM-AF over AS-LLM can be attributed to the unreliability of the algorithm feature source, specifically the inadequate representation of algorithm features in algorithm description texts. This highlights the dependence of AS-LLM on reliable algorithm data, preferably in the form of code files for all candidate algorithms. Additionally, besides demonstrating the performance loss of AS-LLM-FS compared to AS-LLM, the subplot in the top-left corner of Figure \ref{Fig_Ablation} showcases an example of the decreasing loss curves during training epochs for both approaches. This indicates that the feature selection module not only enhances algorithm selection performance but also facilitates easier and faster model convergence. The observed performance loss of AS-LLM-CS indicates the utility of the cosine distance computation, particularly when there is no requirement for generalizing the model to new candidate algorithms. Although the theoretical analyses have not explicitly demonstrated the cosine distance's effectiveness, this empirical result highlights its practical value. Overall, the ablation experiments underscore the necessity of incorporating algorithm features and the utility of each module within AS-LLM.

\section{Conclusion}

This paper explores the importance of algorithm features and their utilization in algorithm selection tasks. Specifically, the proposed AS-LLM leverages the powerful representation capabilities of LLMs to extract algorithm features from text related to code, representing a novel application of LLMs in this context. Additionally, AS-LLM incorporates a feature selection module to identify critical features and utilizes the similarity between algorithm and problem representations for algorithm selection. AS-LLM provides more refined modeling of the bidirectional relationship between algorithms and problems, demonstrating robust performance advantages in diverse scenarios. This paper not only emphasizes the performance superiority of AS-LLM through empirical studies but also provide a rigorous upper bound on its model complexity. This theoretical validation serves as a fundamental basis for model design and practical implementation. Through its theoretical and methodological contributions, we firmly believe that the AS-LLM holds significant application potential in algorithm selection tasks.

While AS-LLM has shown promising results through algorithm representation, its performance might still be unstable in certain situations. In particular, AS-LLM is dependent on the availability of reliable algorithm data, preferably in the form of code files for all candidate algorithms. In scenarios where algorithm code files are not available or the description texts do not adequately represent algorithm characteristics, the model's training and performance could be adversely affected. Moreover, the use of multiple modules in AS-LLM also increases its complexity, which could affect scalability and adaptability to different algorithm selection scenarios. Moving forward, there are exciting possibilities for exploring advanced algorithm representation methods and feature selection techniques. Additionally, considering enhancements to AS-LLM for algorithm selection in specific domains holds promise. Furthermore, delving into the influence of algorithm features on algorithm selection models is a compelling direction for further theoretical investigation.

\section*{Acknowledgments}

This work was supported by the Research Grants Council of the Hong Kong SAR (Grant No. PolyU11211521, PolyU15218622, PolyU15215623,  PolyU25216423, and C5052-23G), The Hong Kong Polytechnic University (Project IDs: P0039734, P0035379, P0043563, and P0046094), and the National Natural Science Foundation of China (Grant No. U21A20512, and 62306259). We also appreciate the comments from anonymous reviewers, which helped to improve the paper.

\bibliographystyle{named}
\bibliography{ijcai24}

\end{document}